%% file: main_arxiv.tex
\documentclass[11pt]{article}

\usepackage[margin=1.0in]{geometry}
\usepackage[utf8]{inputenc}            
\usepackage[T1]{fontenc}               
\usepackage[hidelinks]{hyperref}    	   
\hypersetup{
	colorlinks,
	linkcolor={violet!50!pink},
	citecolor={violet!50!pink},
	urlcolor={violet!50!pink}
}
\usepackage[mathscr]{euscript}
\usepackage[T1]{fontenc}
\usepackage{multirow}

\usepackage{url}                        
\usepackage{booktabs}                   
\usepackage{amsfonts}                   
\usepackage{nicefrac}                   
\usepackage{bbm}
\usepackage{microtype}                  
\usepackage{algorithm}
\usepackage{algpseudocode}
\usepackage[dvipsnames,table]{xcolor} 
\usepackage{lmodern}
\usepackage{comment} 
\usepackage{tikz}
\usepackage[numbers, sort]{natbib}
\usepackage{amsmath, amssymb, amsfonts}
\usepackage{amsthm, thmtools, thm-restate}
\usepackage{mathtools, nccmath, extarrows}
\usepackage[shortlabels]{enumitem} 
\usepackage{subcaption}
\usepackage{bm}
\usepackage{scalerel}
\usepackage{wrapfig}

\makeatletter
\renewcommand{\paragraph}{%
  \@startsection{paragraph}{4}%
  {\z@}{1ex \@plus 1ex \@minus .2ex}{-1em}%
  {\normalfont\normalsize\bfseries}%
}
\makeatother

\usepackage{stackengine}
\mathtoolsset{showonlyrefs} 
\stackMath

\DeclareFontEncoding{LS1}{}{}
\DeclareFontSubstitution{LS1}{stix}{m}{n}

\makeatletter
  \newcommand*\textmathversion{\csname textmv@\math@version\endcsname}
  \newcommand*\textmv@normal{m}
  \newcommand*\textmv@bold{b}
\makeatother

\newtheorem{theorem}{Theorem}[section]
\newtheorem{proposition}[theorem]{Proposition}

\theoremstyle{definition}
\newtheorem{definition}[theorem]{Definition}
\newtheorem{assumption}{Assumption}

\newtheorem{example}[theorem]{Example}

\input{macros}

\begin{document}

\title{Any-dimensional equivariant neural networks}
\author{Eitan Levin\thanks{Department of Computing and Mathematical Sciences, Caltech, Pasadena, CA 91125, USA;\ \texttt{\url{https://eitanlev.in/}}} \and
   Mateo D\'\i az\thanks{Department of Applied Mathematics and Statistics and the Mathematical Institute for Data Science, Johns Hopkins University, Baltimore, MD 21218, USA;\ \texttt{\url{https://mateodd25.github.io/}}}}

\date{June 13, 2023}

\maketitle

\begin{abstract}
\input{abstract}
\end{abstract}

\input{content}
\section*{Acknowledgments}
We thank Venkat Chandrasekaran for his constant support and encouragement during the development of this work. Both authors were funded by Venkat Chandrasekaran through AFOSR grant FA9550-20-1-0320. MD was partially funded by Joel Tropp through ONR BRC Award N00014-18-1-2363, and NSF FRG Award 1952777.

\bibliographystyle{abbrvnat}
\bibliography{biblio}
\appendix
\input{sections/appendix}

\end{document}

%% file: macros.tex
\usepackage{bbold}
\usepackage[most]{tcolorbox}

\newcommand{\cB}{\mathcal{B}}

\newcommand{\cD}{\mathcal{D}}

\newcommand{\cL}{\mathcal{L}}

\newcommand{\cP}{\mathcal{P}}

\newcommand{\cS}{{\mathcal{S}}}

\newcommand{\RR}{{\mathbb R}}
\newcommand{\NN}{{\mathbb N}}

\newcommand{\Tr}{\operatorname{tr}}
\newcommand{\diag}{\operatorname{diag}}

\newcommand{\argmax}{\operatornamewithlimits{argmax}}

\DeclarePairedDelimiter{\dotp}{\langle}{\rangle}

\def\shortdisplay{\setlength{\abovedisplayskip}{5pt}%
	\setlength{\belowdisplayskip}{5pt}%
	\setlength{\abovedisplayshortskip}{2pt}%
	\setlength{\belowdisplayshortskip}{2pt}}

\let\oldselectfont\selectfont
\def\selectfont{\oldselectfont\shortdisplay}

\newcommand{\Vn}{{\bm V}}
\newcommand{\Vnp}{{\bm V}^+}

\newcommand{\Un}{{\bm U}}
\newcommand{\Unp}{{\bm U}^+}

\newcommand{\Gn}{{\bm G}}
\newcommand{\Gnp}{{\bm G}^+}
\newcommand{\sn}{{\bm \sigma}}
\newcommand{\snp}{{\bm \sigma}^+}

\newcommand{\bn}{{\bm b}}
\newcommand{\bnp}{{\bm b}^+}

\newcommand{\Wn}{{\bm W}}
\newcommand{\Wnp}{{\bm W}^+}

\newcommand{\fn}{{\bm f}}
\newcommand{\fnp}{{\bm f}^+}

\newcommand{\Unn}{U^{(n)}}

\newcommand{\Vnn}{V^{(n)}}
\newcommand{\Vnnp}{V^{(n+1)}}
\newcommand{\Gnn}{G^{(n)}}
\newcommand{\Gnnp}{G^{(n+1)}}
\newcommand{\Fnn}{F^{(n)}}

\newcommand{\Knn}{K^{(n)}}
\newcommand{\Knnp}{K^{(n+1)}}
\newcommand{\Wnn}{W^{(n)}}
\newcommand{\Wnnp}{W^{(n+1)}}

\newcommand{\phin}{\bm{\varphi}}
\newcommand{\psin}{\bm{\psi}}

\definecolor{block-gray}{gray}{0.93}
\newtcolorbox{myquote}{colback=block-gray,
,grow to right by=-10mm,grow to left by=-10mm,
boxrule=0pt,boxsep=0pt
}

\newcommand{\mc}[1]{\mathcal{#1}}
\newcommand{\mbb}[1]{\mathbb{#1}}
\newcommand{\mscr}[1]{\mathscr{#1}}

\newcommand{\gO}{\mathrm{O}}
\newcommand{\gSO}{\mathrm{SO}}

%% file: abstract.tex
Traditional supervised learning aims to learn an unknown mapping 
by fitting a function 
to a set of input-output pairs with a \emph{fixed dimension}. The fitted function 
is then defined on inputs of the same dimension. However, in many settings, the unknown mapping 
takes inputs in \emph{any dimension}; examples include graph parameters defined on 
graphs of any size 
and physics quantities defined on an arbitrary number of particles. 
We leverage a newly-discovered phenomenon in algebraic topology, called representation stability, to define equivariant neural networks that can be trained with data in a fixed dimension and then extended to accept inputs 
in any
dimension.
Our approach is black-box and user-friendly, requiring only the network architecture and the groups for equivariance, and can be combined with any training procedure.
We provide a  
simple open-source
implementation of our methods
and 
offer preliminary numerical experiments. 

%% file: content.tex
\input{sections/intro}

\input{sections/free}

\input{sections/compatibility}

\input{sections/computational-framework}

\input{sections/experiments}

\input{sections/limitations}

\input{sections/conclusions}


%% file: sections/intro.tex
\section{INTRODUCTION}\label{sec:intro}

Researchers are often interested in learning mappings that are defined for inputs of different sizes.
We call such objects ``\emph{free}''
as they are dimension-agnostic. Free objects are pervasive in a range of scientific and engineering fields. For example, in physics, quantities like electromagnetic fields are defined for any number of particles. In signal processing, problems such as deconvolution are well-defined for signals of any length, while image and text classification tasks are meaningful regardless of the input size. In mathematics, norms of vectors and matrices are defined for any dimension; sorting algorithms 
handle arrays of any length, and graph parameters, such as the max-cut value, are defined for graphs of any size.

In contrast, most existing supervised learning methods aim to learn mappings by fitting a parametrized function $\widehat f \colon \RR^{d} \rightarrow \RR^{k}$ to a set of input-output pairs
$\cS = \{(X_{1}, y_{1}), \dots, (X_{n}, y_{n})\} \subseteq \RR^{d} \times \RR^{k}$ with fixed dimensions $d$ and $k$. Naturally, the function $\widehat f$ is only defined on inputs of dimension $d$, and \textit{a priori} cannot be applied to inputs in other dimensions. Practitioners use ad-hoc heuristics, such as downsampling, to resize the dimension of new inputs and match that of the learned map \citep{hashemi2019enlarging}. 
The lack of systematic methodology for applying learned functions to inputs of different sizes motivates the main question of this work:
\begin{myquote} \centering \it
  How can we learn mappings that accept inputs in any dimension?
\end{myquote}
We answer this question for mappings that are invariant or equivariant under the action of groups.
Such mappings are ubiquitous since many application domains exhibit symmetry,
e.g., physical quantities are equivariant under translations and rotations because they only depend on the relative position of particles~\citep{villar2021scalars}, while graph parameters are equivariant under permutations because they only depend on the underlying 
topology
rather than on its labelling~\citep{chandrasekaran2012convex}.
We propose to learn such mappings using free equivariant neural networks.
Formally, we consider a sequence of neural networks and a sequence of groups, such that a network at any level in the sequence is equivariant with respect to the group at the same level; see Figure~\ref{fig:freeNNs}. 
We train an equivariant network with data $\cS$ in a fixed dimension and leverage symmetry to extend it 
to other dimensions.

\paragraph{Core contributions.} 
The question posed above presents three key challenges.
First, how do we encode infinite sequences of neural networks, one in each dimension, using only finitely-many parameters? Doing so is necessary if we want to learn from a finite amount of data.
Second, how do we ensure that the networks we learn from data in a fixed dimension generalize well to higher dimensions? Third, is there a user-friendly procedure for learning these networks?
We proceed to tackle these challenges.
%

\textit{Free equivariant neural networks.} Equivariance is the key ingredient in tackling the first challenge. 
The authors of~\citep{maron2018invariant} showed that the dimension of permutation-equivariant linear layers between tensors is independent of the size of the tensors, and gave free bases for such linear layers. For example, the space of permutation-equivariant linear layers for graph adjacency matrices is 15-dimensional for any $n$, and its basis contains the identity, taking a transpose and extracting the diagonal, among others.
The authors of that paper use this fact to derive finite parameterizations of linear layers
\begin{equation}
  \label{eq:1}
  W = \alpha_{1} A_{1}^{(n)} + \dots + \alpha_{15} A_{15}^{(n)}
\end{equation}
 where $\left\{A_{1}^{(n)}, \dots, A_{15}^{(n)}\right\}$ is a basis of free equivariant maps for graphs on $n$ nodes. This parameterization allows them to instantiate linear layers in any dimension using only the fifteen parameters $\alpha_i$, which can be learned from data in a fixed dimension. 
Our \textbf{first contribution} is showing that this is not a coincidence but rather stems from a general, recently-identified phenomenon known as \emph{representation stability}. We utilize this phenomenon to show that the dimension of equivariant linear layers stabilizes for a number of group actions, including those induced by (signed) permutations, the orthogonal groups $\gO(n)$, rotations $\gSO(n)$, and the Lorentz groups $\gO(1,n)$.
We leverage this observation to derive finite parametrizations of infinite sequences of equivariant neural networks.


\textit{Generalization across dimensions.} The authors of~\citep{maron2018invariant} observed that neural networks obtained using~\eqref{eq:1} do not always generalize well to dimensions different from the one used for training. Intuitively, this is due to overparametrization, as there are many ways to represent the same function in a fixed dimension, but not all of them extend correctly to other dimensions. 
Our \textbf{second contribution} is to introduce a compatibility condition relating the mappings in different dimensions, which has a regularization effect that often leads to better generalization.
Formally, it entails to the commutativity of the diagrams in Figure~\ref{fig:freeNNs}. 
We further explain how to impose this condition on the network architecture.

\textit{Computational recipe.} The authors of~\citep{maron2018invariant} \emph{manually} found the free basis elements in~\eqref{eq:1}.
However, for more complicated groups and spaces of linear layers, manually finding a free basis becomes prohibitive.
In a fixed dimension, one can circumvent this issue using an algorithm proposed in~\citep{pmlr-v139-finzi21a} to compute a basis.
For our \textbf{third contribution}, we 
extend this basis to any other dimension by solving a sparse linear system, enabling us to obtain free bases computationally. 

We provide proof-of-concept experiments demonstrating the ability of our approach to learn functions defined for any dimension. We emphasize, however, that the primary contribution of this paper is its conceptual novelty. To our knowledge, our method is the first \emph{black-box approach} for training neural networks that can handle inputs in any dimension.
Instead of customizing architectures for specific domains and data types, as has been done in the prior work discussed below, our strategy adapts itself to the application domain by selecting appropriate group actions and nested sequences of vector spaces. 
The downside of the more general approach is that we do not expect it to outperform application-specific methods. 
Thus, we envision our framework serving as an initial step when approaching novel and poorly understad problems, or serving as a baseline to evaluate the performance of application-specific architectures.

\paragraph{Notation.} \label{par:notation} Given two finite-dimensional vector spaces $U$ and $V$ we use $\cL(V, U)$ to denote the set of linear mappings between them. For a group $G$ acting on both $U$ and $V$, we denote invariant elements as $U^G= \{u \in U \mid g\cdot u = u\}$, and equivariant linear maps as $\cL(V, U)^G = \left\{ A \in \cL(V, U) \mid g A  g^{-1} = A\right\}.$ The map $\mc P_W\colon U\to U$ denotes the projection onto a subspace $W \subseteq U$. The symbols $\mathbbm{1}_n$ and $I_n$ represent the all-ones vector and the identity in dimension $n$, respectively. The symbol $\RR[G]$ denotes the set of finite linear combinations of elements in $G$. To avoid cluttering notation, we use bold font symbols whenever possible to denote the $n$th element in a sequence, e.g., the symbol $\Gn$ denotes $G^{(n)}$. We add a ``$+$'' superscript to denote the $(n+1)$th element, e.g., $G^{(n+1)}$ is denoted by $\Gn^+$. 

\paragraph{Outline.} The remainder of this section focuses on related work. Section~\ref{sec:free-descriptions} formally defines free neural networks. Section~\ref{sec:compatibility} introduces a compatibility condition that ensures good generalization across different dimensions. In Section~\ref{sec:computational}, we describe our computational recipe to learn free neural networks from data in a fixed dimension and extend to other dimensions. Section~\ref{sec:experiments} provides numerical experiments. We close the paper in Section~\ref{sec:conclusions} with conclusions, limitations, and future work. We defer several technical details to the appendix.

\subsection{Related work}
\paragraph{Equivariant learning.}
The benefits of equivariant architectures first became apparent with the success of convolutional neural networks (CNNs) in computer vision~\citep{cnns_imagenet}. 
Since then, equivariance has been applied to a range of applications. Examples include DeepSets~\citep{deepsets} and graph neural networks~\citep{GNNs,maron2018invariant} using permutation equivariance to process sets and graphs; AlphaFold 2~\citep{alphafold} and ARES~\citep{ARES} using $\mathrm{SE}(3)$-equivariance for protein and RNA structure prediction and assessment; steerable~\citep{cohen2017steerable} and spherical~\citep{sphericalCNNs} CNNs using rotation-equivariance to classify images; and physics-informed neural networks are equivariant under the symmetries of the corresponding physical systems~\citep{karniadakis2021physics}. 
Many of these architectures have been shown to implement a generalized notion of convolution over the groups at hand~\citep{pmlr-v80-kondor18a, pmlr-v48-cohenc16}.
Under additional assumptions, equivariant architectures have been derived using invariant theory to explicitly parametrize polynomial equivariant maps~\citep{villar2021scalars}.
We refer the reader to~\citep{whatIs,bronstein2021geometric} for an introduction to equivariant deep learning.

\paragraph{Dimension-free learning.}
Certain architectures processing data of a particular type are defined for inputs of different sizes.
Many architectures processing graph-based data update the features at each vertex by applying the same function of the features of the vertex's neighbors, hence can be applied to graphs of any size~\citep{pmlr-v70-gilmer17a,SGNNs,GNNs}.
Convolutional neural networks (CNNs) processing signals and images convolve their inputs with filters of constant size, hence can also be applied to inputs of different sizes. Nevertheless, it has been observed that na\"ively applying CNNs to inputs of size different from that used during training leads to artifacts, and several downsampling techniques have been proposed to resize inputs to a CNN~\citep{hashemi2019enlarging}. 
Networks processing natural language embed words as vectors of the same length but are defined for arbitrarily long sequences of such vectors. 
Recurrent neural networks process such sequences one-by-one, using cycles in their architectures to sequentially combine a new input vector with a function of the previous inputs~\citep{rnns}. 
Recursive neural networks apply the same weights recursively to pairs of input vectors to combine them into one, until the entire sequence is reduced to one vector~\citep{socher2011parsing}.
Notably, all these architectures must process the input in particular ways motivated by the specific application at hand, whereas we only need to assume that the network processes its input equivariantly.

\paragraph{Representation stability.}
Representation stability considers nested sequences of groups and their representation, and implies that such sequences of representations often stabilize. Specifically, there is a labelling of the irreducibles of these groups such that the decomposition of the representations in the sequence contain the same irreducibles with the same multiplicities. This phenomenon has been formalized in~\citep{CHURCH2013250} and further studied in~\citep{FImods,WILSON2014269,GADISH2017450,sam_snowden_2015,sam2016gl,sam2017grobner}. 
It has been applied to study free convex sets and algebraic varieties~\citep{levin2023free,van2021theorems,Draisma2014,chiu2022sym,alexandr2023moment}, though to our knowledge, this paper is the first to apply it to equivariant deep learning.
We refer the reader to~\citep{farb2014representation,wilson2018introduction,sam2020structures} for introductions to this area.


%% file: sections/free.tex
\section{FREE NEURAL NETWORKS}\label{sec:free-descriptions}
\begin{figure}
    \centering
    \includegraphics[width=.8\linewidth]{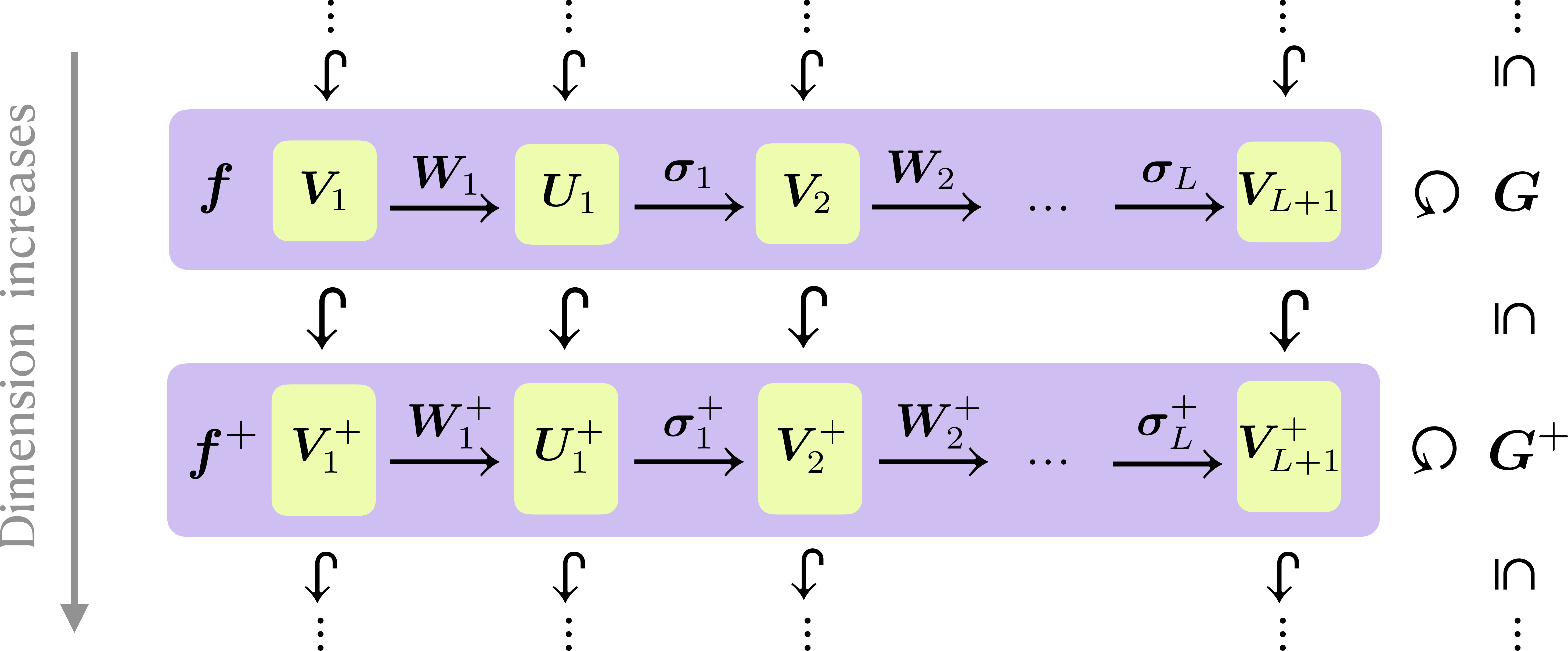}
    \caption{Equivariant free neural networks. We use bold font to denote the $n$th element in a sequence; see Notation in Section~\ref{par:notation} for details. 
    }
    \label{fig:freeNNs}
\end{figure}
In this section, we introduce the concept of free neural networks, i.e., networks that can be instantiated in every dimension. 
To set the stage, let us recall the classical notion of a neural network (NN). A NN is a mapping $f=f_L\circ\ldots\circ f_1$ where $f_i\colon V_i\to V_{i+1}$ is a composition $f_i(x) = \sigma_i(W_ix+b_i)$ of an affine map $x\mapsto W_ix+b_i$ and an activation map $\sigma_i\colon U_i\to V_{i+1}$.
This yields a family of mappings parametrized by the weights $W_i\in\mc L(V_i,U_i)$ and biases $b_i\in U_i$. In turn, these parameters are chosen to minimize a prescribed loss function.

In several settings, we want $f$ to respect symmetries in its inputs.
Formally, we require $f$ to be equivariant with respect to the action of a group $G$, i.e., $f\circ g=g\circ f$ for all $g\in G$.
A natural way to obtain equivariance is by ensuring that each building block of $f$ is equivariant, i.e., $W_i \in \cL(V_i, U_i)^G$, $b_i \in U_i^G$, and $\sigma_i$ is $G$-equivariant. 
NNs satisfying these properties are called \emph{equivariant}. Equivariant NNs can be trained by computing a basis for the space of weights and biases and optimizing over the coefficients in this basis~\citep{pmlr-v139-finzi21a}.

In this work, we seek equivariant NNs that extend to inputs in \emph{any dimension}. To this end, we consider sequences $\{\Un_i\}$ and $\{\Vn_i\}$ of nested\footnote{Formally, we have embeddings $\Vn \hookrightarrow \Vnp$ and $\Un \hookrightarrow \Unp$.} vector spaces $\Un_i \subseteq \Unp_i$ and $\Vn_i \subseteq \Vnp_i$, with actions of a nested sequence of groups $\{\Gn\}$, see Figure~\ref{fig:freeNNs}. For example, we have inclusions $\RR^n\subseteq \RR^{n+1}$ by zero-padding on which the group of permutations on $n$ letters $\Gn=\mathrm{S}_n$ acts by permuting coordinates.
These define a sequence of NNs whose weights $\Wn_i\colon \Vn_i \rightarrow \Un_i$ and biases $\bn_i\in \Un_i$ increase in size. The activation functions $\sn$ are often defined for every dimension, as the next section shows. Thus, we focus here on extending the linear layers.


\textbf{Free equivariant networks.} Because the dimensions of the space of linear layers $\cL(\Un_i,\Vn_i)$ and the vector spaces $\Un_i$ usually grow, sequences of \emph{general} NNs are not finitely-parameterizable and cannot be learned from a finite amount of data.  The situation radically simplifies when considering equivariant networks since the dimensions of $\cL(\Un_i,\Vn_i)^{\Gn}$ and $\Un_i^{\Gn}$ often stabilize. This was previously observed in the context of simple graph NNs by~\citep{maron2018invariant}, and follows from a general phenomenon known as representation stability~\citep{CHURCH2013250}.
To illustrate it with a concrete example, consider the nested sequence $\Un=\Vn=\RR^n$ with the action of $\Gn=\mathrm{S}_n$ as above. Then $\Vn^{\Gn}$ is one-dimensional and spanned by $\mathbbm{1}_n$ while $\cL(\Un,\Vn)^{\Gn}$ is two-dimensional and spanned by $I_n, \mathbbm{1}_n\mathbbm{1}_n^\top$, which are basis elements defined for every $n$. 
Similarly, \citep{maron2018invariant} obtained explicit free bases for spaces of invariant tensors. 

Interestingly, all the above basis elements project onto each other, e.g., the orthgonal projection of $\mathbbm{1}_n$ onto $\RR^{n-1}$ is $\mathbbm{1}_{n-1}$, and similarly $I_n, \mathbbm{1}_n\mathbbm{1}_n^\top$ project onto $I_{n-1}, \mathbbm{1}_{n-1}\mathbbm{1}_{n-1}^\top$. Motivated by this observation, we say that a sequence of equivariant NNs is \emph{free} if the weights $\Wn_i\colon \Vn_i\to \Un_i$ and biases $\bn_i\in \Un_i$ satisfy
\begin{equation}\tag{\texttt{FreeNN}}\label{eq:freeNN}
    \Wn_i = \mc P_{\Un_i}\left.\Wnp_i \right|_{\Vn_i},\qquad \text{and} \qquad \bn_i = \mc P_{\Un_i}\bnp_i.
\end{equation}
For many sequences $\{\Vn_i\}, \{\Un_i\}$, equation~\eqref{eq:freeNN} uniquely determines $\Wnp_i, \bnp_i$ from $\Wn_i,\bn_i$, allowing us to uniquely extend a network to accept larger inputs.

\textbf{Representation stability.} To understand for which sequences~\eqref{eq:freeNN} allows us to extend, we introduce some key definitions from the representation stability literature~\citep{CHURCH2013250,FImods}.
\begin{definition}[Consistent sequences]\label{def:consistent_seqs}
Fix a sequence of compact groups $\mscr G=\{\Gn\}_{n\in\NN}$ such that $\Gn \subseteq \Gnp $.
The family $\mscr V = \{(\Vn,\phin)\}_{n\in\NN}$ is a \emph{consistent sequence} of $\mscr G$-representations if the following hold true for all $n$:
\begin{enumerate}[(a)]
    \item (\textbf{Representations}) The set $\Vn$ is an orthogonal $\Gn$-representation;
    \item (\textbf{Equivariant isometries}) The map $\phin \colon \Vn \hookrightarrow \Vnp$ is a $\Gn$-equivariant isometry.
\end{enumerate}
\end{definition}
\noindent We will identify $\Vn$ with its image $\phin(\Vn) \subseteq \Vnp$, and omit the inclusions $\phin$ unless needed. 
Importantly, consistent sequences can be combined to form more complex sequences. In particular, if $\mscr V = \{(\Vn,\phin)\}$ and $\mscr U = \{(\Un,\psin)\}$ are consistent sequences of $\{\Gn\}$-representations, then so are their sum and tensor product:
    \begin{equation}
        \mscr V\oplus \mscr U = \{(\Vn\oplus \Un, \phin\oplus\psin)\},\quad
        \mscr V\otimes\mscr U = \{(\Vn\otimes \Un, \phin\otimes\psin)\}. 
    \end{equation}
This follows directly from Definition~\ref{def:consistent_seqs}.
 We use $\mscr V^{\oplus k}$ (resp. $\mscr V^{\otimes k}$) to denote the direct sum (resp., tensor product) of $\mscr V$ taken $k$ times. As a direct by-product of this observation, we obtain that the sequence of linear maps $\cL(\mscr V,\mscr U) = \{(\mc L(\Vn,\Un), \phin\otimes\psin)\} $ is consistent since it is isomorphic to $\mscr V\otimes\mscr U$. 
The following parameter controls the complexity of consistent sequences and ensures that their spaces of invariants are eventually isomorphic.
\begin{definition}[Generation degree]\label{def:gen_deg}
    A consistent sequence $\{\Vn\}$ 
    is \emph{generated in degree $d$} if $\RR[\Gn]V^{(d)}=\Vn$ for all $n\geq d$. The generation degree is the smallest such $d \in  \NN$. The sequence is \emph{finitely-generated} 
    if the generation degree is finite.
\end{definition}
\noindent In words, a consistent sequence $\{\Vn\}$ is generated in degree $d$ if for all $n \geq d$, $\Vn$ is equal to the span of linear combinations of elements of $\Gn$ applied to $V^{(d)}$. For instance, if $\Vn=\RR^d$ with the action of permutations $\Gn=\mathrm{S}_n$, then its generation degree is one, as any $v \in \Vn$ can be obtained from $e_1 = (1, 0,\dots, 0) \in V^{(1)} \subseteq \Vn$ via $v = \sum v_i \, g_i \cdot e_1$ where $g_i$ swaps the the first and $i$th components. 
The next proposition gives an isomorphism between spaces of invariants in a finitely-generated consistent sequence.
This result was first established in the representation stability literature, albeit in a different form. We include proof for completion.
\begin{proposition}[Isomorphism of invariants]\label{prop:surj_imp_inj}
If $\{\Vn\}$ is generated in degree $d$, then the orthogonal projections $\mc P_{\Vn}\colon (\Vnp)^{\Gnp}\to \Vn^{\Gn}$ are injective for all $n\geq d$, and isomorphisms for all large $n$.
\end{proposition}
\begin{proof}
First, the projection $\mc P_{\Vn}$ is $\Gn$-equivariant because $\Gn$ acts orthogonally. Second, it maps $\Gnp$-invariants in $\Vnp$ to $\Gn$-invariants in $\Vn$ because $\Gn\subseteq \Gnp$. 
Third, we prove injectivity for $n\geq d$. Suppose $\mc P_{\Vn}v=0$ for some $v\in (\Vnp)^{\Gnp}$. For any $u\in \Vnp$, write $u = \sum_ig_i u_i$ where $u_i\in \Vn$ and $g_i\in \Gnp$. We then have $\langle v, u\rangle = \langle v, \sum_iu_i\rangle = 0$ because $v$ is invariant and $\sum_iu_i\in \Vn$. Since $u\in \Vnp$ was arbitrary, we conclude that $v=0$. The injectivity of $\mc P_{\Vn}$ shows that $\dim (\Vn)^{\Gn}\geq \dim (\Vnp)^{\Gnp}$ for all $n\geq d$, hence the sequence of dimensions $\dim (\Vn)^{\Gn}$ eventually stabilizes, at which point the projections become isomorphisms.
\end{proof} 
Note that~\eqref{eq:freeNN} precisely requires $\Wnp$ and $\bnp$ to project to $\Wn,\bn$ inside the consistent sequences $\mc L(\mscr V_i,\mscr U_i),\mscr U_i$, respectively.
Thus, Proposition~\ref{prop:surj_imp_inj} implies that for all large $n$, we can parameterize infinite sequences of equivariant linear layers with a finite set of parameters $\alpha_1, \dots, \alpha_\ell$, just as in \eqref{eq:1}. This result enables us to train free NNs as we describe in Section~\ref{sec:computational}. 


\textbf{Examples.} \label{par:examples} We collect a few crucial examples and include additional ones in the appendix.
\begin{itemize}[leftmargin=2mm]
 \item[] \textit{Scalar sequence.} The sequence $\mscr S=\{\RR\}$ together with $\phin (x) = x$ yields a consistent sequence for any sequence of groups acting trivially, i.e., $g \cdot x = x$. It is generated in degree one. By setting the last layer $\Vn_{L+1} = \RR$ in the free NN architecture of Figure~\ref{fig:freeNNs}, we obtain a sequence of invariant functions $\fn$.

 \item[] \textit{Permutation sequences.} Let $\mscr R = \{\RR^n\}$ be the consistent sequence with zero-padding embeddings $\phin(x)=(x^\top,0)^\top$ and the action of the symmetric group $\Gn=\mathrm{S}_n$ by permuting coordinates. Consider $\mscr V_k = \mscr U_{k-1} = \mscr R^{\otimes k}$, which is generated in degree $k$. 
The dimensions of the space of weights and biases for low-order tensors are eventually
\begin{equation*}\begin{aligned}
    &\dim\mc L(\Vn_1,\Un_1)^{\Gn}=4,\qquad &&\dim \Un_1^{\Gn} = 2,\\ &\dim \mc L(\Vn_2,\Un_2)^{\Gn}=52,\qquad &&\dim \Un_2^{\Gn}=
    5.
\end{aligned}\end{equation*}

 \item[] \textit{Graph(on) sequences.} There are two ways to identify small graphs with larger ones, namely, appending isolated vertices and identifying graphs with their associated graphons. These two embeddings yield two consistent sequences. 
 For the first embedding, let $\Vn = \mbb S^n$ with the Frobenius inner product and embeddings $\phin(X) = \mathrm{blkdiag}(X,0)$ padding $X$ by a zero row and zero column. We endow $\Vn$ with the action of $\Gn=\mathrm{S}_n$ by simultaneously permuting rows and columns (so $g\cdot X=gXg^\top$ for all $g\in\Gn$ and $X\in\Vn$). This sequence is generated in degree 2.
 For the second embedding, let $\Vn = \mbb S^{2^n}$ with the normalized inner product $\langle X,Y\rangle=2^{-2n}\mathrm{Tr}(XY)$, embeddings $\phin(X)=X\otimes\mathbbm{1}_{2\times 2}$, and the same action of the symmetric group $\Gn=\mathrm{S}_{2^n}$ as above. 
 This sequence arises in the theory of graphons~\cite{lovasz2012large}, where adjanceny matrices $X$ are associated with a step function (the graphon) $W_X\colon [0,1]^2\to\RR$ in such a way that $X$ and $X\otimes\mathbbm{1}_{2\times 2}$ represent the same graphon and $\langle X,Y\rangle = \int_{[0,1]^2}W_X(x,y)W_Y(x,y)\, dx\, dy$. The graphon consistent sequence is generated in degree 2 by~\cite[Prop.~3.1]{levin2023free}.
\end{itemize}

\textbf{When does stability kick in?} Understanding the exact level at which the projections become isomorphisms is important, yet, an exact characterization is rather technical.
Stabilization occurs at the so-called \textbf{presentation degree} of the sequence, which might be larger than the generation degree in general. However, they agree for many relevant examples, such as in the sequence $\mscr R^{\otimes k}$ above.
We include a formal definition and a discussion in Appendix~\ref{sec:presentation-degree}.

%% file: sections/compatibility.tex
\section{GENERALIZATION ACROSS DIMENSIONS AND COMPATIBILITY}\label{sec:compatibility}

Free NNs often do not generalize correctly to other dimensions. In Section \ref{sec:experiments}, we provide examples where a free NN yields excellent test error in its trained dimension while exhibiting errors $10^{10}$ times worst in other dimensions. This discrepancy arises due to overparametrization, as there are numerous ways to encode the same function in a fixed dimension, and yet not every encoding extends seamlessly. In this section, we propose a regularization strategy 
that often leads to better generalization across dimensions.

Specifically, inspired by a similar condition for free convex sets 
\citep{levin2023free}, we introduce the following compatibility condition.

\begin{definition}[Compatible networks]
A sequence of maps $\{\fn\colon \Vn\to \Un\}$ is (intersection-) \emph{compatible} if $\fnp|_{\Vn}=\fn$. A sequence of equivariant networks
(Figure~\ref{fig:freeNNs})
is compatible if 
\begin{equation}\tag{\texttt{CompNN}}\label{eq:compatibleNN}
    \Wnp_i|_{\Vn_i} = \Wn_i,\quad \bnp_i = \bn_i,\quad \text{and} \quad \snp_i|_{\Un_i} = \sn_i.
\end{equation}
\end{definition}
\noindent Intuitively, this condition ensures that applying a function in a fixed dimension and then extending to a higher dimension is equivalent to first extending and then applying the function. Layer-wise compatibility \eqref{eq:compatibleNN} guarantees that the sequence of NN $\{\fn\colon \Vn_1\to \Vn_{L+1}\}$, given by Figure~\ref{fig:freeNNs}, is compatible, or equivalently, that the diagram in that figure commutes.
The condition \eqref{eq:compatibleNN} is strictly stronger than \eqref{eq:freeNN}, as compatible NNs are free, but free NNs are not compatible in general. 

\paragraph{Compatibility in the wild.} \label{ex:comp} Compatibility is a natural condition satisfied by sequences of equivariant maps $\{\fn\}$ arising in many problems of interest.
Here we list some examples. 
\begin{enumerate}[leftmargin=2mm]
\item[] \textit{Graph(on) parameters.} Graph parameters are functions of graphs (typically of any size) that do not depend on the labeling of their vertices. In other words, these are sequences of $\mathrm{S}_n$-invariant functions \{$\fn\colon\mbb S^n\to \RR\}$. Several graph invariants are also compatible with zero-padding, which is equivalent to adding an isolated vertex. For example, the max-cut value, the number of triangles, and cycles do not change if we add an isolated vertex, hence they are compatible. 
Other graph invariants only depend on a graph via its associated graphon, and hence are compatible with the embeddings in the graphon sequence (see example at the end of Section~\ref{sec:free-descriptions}). These include graph homomorphism densities, which play a central role in extremal combinatorics~\cite{lovasz2012large}.


\item[] \textit{Matrix mappings.} Linear algebra operations are often compatible. For instance, the experiments in ~\citep{maron2018invariant} aimed to learn the following matrix mappings: 
$X\mapsto \frac{1}{2}(X+X^\top)$, 
$X\mapsto \diag(\diag(X))$, 
$X \mapsto \Tr(X)$, and 
$X \mapsto \argmax_{\|v\|=1}\|Xv\|$. All of these are compatible sequences with the zero-padding embedding described in Section \ref{par:examples}, and $\mathrm{S}_n$-equivariant.

    \item[] \textit{Orthogonal invariance.} The papers \citep{pmlr-v139-finzi21a,villar2021scalars} consider the task of learning the function
    \begin{equation}\label{eq:sin}
        \fn(x_1,x_2) = \sin(\|x_1\|) - \frac{\|x_2\|^3}{2} + \frac{x_1^\top x_2}{\|x_1\|\|x_2\|},
\end{equation}
    defined on $(\RR^n)^{\oplus 2}$ with $n=5$ from evaluation data. This function is $\mathrm{O}(n)$-invariant if we let rotations $g$ act by $(g \cdot x_1, g \cdot x_2)$. Embedding $(\RR^n)^{\oplus 2}$ into $(\RR^{n+1})^{\oplus 2}$ by zero-padding each vector, we see that $\{\fn\}$ is a compatible sequence of invariant functions.

    \item[] \textit{Orthogonal equivariance.} The $\mathrm{O}(3)$-equivariant task in~\citep{pmlr-v139-finzi21a,villar2021scalars} consists of taking as input $n=5$ particles in space, given by their masses and position vectors $(m_i,x_i)_{i=1}^n\in (\RR\oplus\RR^3)^{\oplus n}$, and outputting their moment of inertia matrix 
    $
        \fn(m_i,x_i)= \sum_{i=1}^nm_i(x_i^\top x_iI_3 - x_ix_i^\top).
    $
    Embedding $\Vn=(\RR\oplus\RR^3)^{\oplus n}$ into $\Vnp$ by zero-padding and letting $\Un=\mbb S^3$, we see that $\{\fn\}$ is a compatible sequence of maps. Furthermore, let $\Gn=\mathrm{S}_n\times \mathrm{O}(3)$ act on $\Vn$ by $(\pi,g)\cdot (m_i,x_i)_{i=1}^n = (m_{\pi^{-1}(i)},gx_{\pi^{-1}(i)})_{i=1}^n$, and on $\Un$ by $(\pi,g)\cdot X=gXg^{-1}$ for $(\pi,g)\in \Gn$ (so permutations act trivially). Then each $\fn$ is $\Gn$-equivariant.

\end{enumerate}
There are many more examples of compatible mappings in the literature \citep{levin2023free}. Therefore, we proceed to study compatible NNs. We derive conditions ensuring the compatibility of the linear layers and show that several standard activation functions are compatible.

\paragraph{Compatible linear layers.}
Since we only have data in some finite level $n_0$, we ask when do fixed-dimensional weights $W^{(n_0)}_i$ extend to a sequence satisfying~\eqref{eq:compatibleNN}. 
The set of such $W_i^{(n_0)}$ forms a linear space, and the next theorem characterizes a subspace of it and enables us to find a basis for that subspace in Section~\ref{sec:computational}. We defer its proof to Appendix~\ref{sec:proof-compatibility}. 
\begin{assumption} \label{ass:comp} The sequences $\mscr V=\{\Vn\},\ \mscr U=\{\Un\}$ are obtained from direct sums and tensor products of the same sequence $\mscr V_0$. The sequence $\mscr V$ is generated in degree $d_{\rm g}$ and presented in degree $d_{\rm p}$. The mapping $W^{(n_0)}\in \mc L(V^{(n_0)},U^{(n_0)})^{G^{(n_0)}}$ at level $n_0\geq d_{\rm p}$ satisfies
$$ W^{(n_0)}(V^{(m)})\subseteq U^{(m)} \textrm{ for } m\leq d_{\rm g}.$$
\end{assumption}
\begin{theorem}\label{thm:ext_to_morph}
Suppose that the sequences $\mscr V=\{\Vn\},\ \mscr U=\{\Un\}$ and linear map $W^{(n_0)}\in \mc L(V^{(n_0)},U^{(n_0)})^{G^{(n_0)}}$ satisfy Assumption~\ref{ass:comp}. Then there is a unique extension $\{\Wn\}$ of $W^{(n_0)}$ to a sequence of equivariant linear maps satisfying~\eqref{eq:compatibleNN}. 
\end{theorem}
Recall that we define the presentation degree in the appendix. 
In words, Theorem~\ref{thm:ext_to_morph} says that if $W^{(n_0)}$ is equivariant and its restrictions to lower-dimensional subspaces satisfy~\eqref{eq:compatibleNN}, then it uniquely extends to higher-dimensional weights satisfying~\eqref{eq:compatibleNN}.
In Section \ref{sec:computational}, we use this result to find a basis for such $W^{(n_0)}$, and use it to train compatible equivariant NNs.

\paragraph{Compatible activation functions.}\label{sec:comp_acti}
The majority of equivariant nonlinearities proposed in the literature are compatible, such as the following few examples. 
\begin{enumerate}[leftmargin=2mm]
     \item[] \textit{Entrywise activations and permutations.} Let $\sigma\colon\RR\to\RR$ be any nonlinear function. Define $\sn\colon (\RR^n)^{\otimes k}\to (\RR^n)^{\otimes k}$ by applying $\sigma$ to each entry of a tensor. Then each $\sn$ is $\Gn=\mathrm{S}_n$-equivariant, and $\{\sn\}$ is a compatible sequence with respect to the zero-padding embeddings if and only if $\sigma(0)=0$. 
     On the other hand, if $\sn\colon(\mbb S^{2^n})^{\otimes k}\to (\mbb S^{2^n})^{\otimes k}$ is again applied entrywise and we use the graphon embedding (see example at the end of Section~\ref{sec:free-descriptions}), then $\{\sn\}$ is a compatible sequence for \emph{any} $\sigma$.
    
    \item[] \textit{Bilinear layers}. The authors of \citep{pmlr-v139-finzi21a} map $(\Vn\otimes \Un)\oplus \Vn \to \Un$ by sending $(v\otimes u,v')\mapsto \langle v,v'\rangle u$, which is equivariant for \emph{any} group action on $\Vn,\Un$. This sequence of maps is also compatible.

    \item[] \textit{Gated nonlinearities.} The authors of \citep{weiler20183d} map $\Vn\oplus\RR\to \Vn$ by $(v,\alpha)\mapsto v\sigma(\alpha)$ where $\sigma\colon\RR\to\RR$ is a nonlinearity. These are equivariant maps for any group acting trivially on the $\RR$ component, and they form a compatible sequence.
    \end{enumerate}

%% file: sections/computational-framework.tex
\section{A COMPUTATIONAL RECIPE FOR LEARNING FREE NEURAL NETWORKS}\label{sec:computational}

In this section, we describe an algorithm to train free and compatible NNs.
We do so in two stages. First, we train our NN at a large enough level $n_0$ by finding bases for the weights and biases at that level and optimizing over the coefficients in this basis. 
Second, for any higher level $n>n_0$, we extend the trained NN at level $n_0$ by solving linear systems for the higher-dimensional weights and biases; for any lower level $n < n_0$ we project the NN using~\eqref{eq:freeNN}. 
Throughout, we fix consistent sequences $\mscr V_i,\mscr U_i$ and sequences of compatible equivariant nonlinearities $\{\sn_i\}$. 
We summarize our procedure in Algorithm~\ref{algo:free_descr}. In Appendix~\ref{sec:correctness}, we show that our procedure provably generates free and compatible NN.

\textbf{Finding bases for free NN.} To train free networks~\eqref{eq:freeNN}, we fix $n_0$ \emph{exceeding the presentation degrees} of $\mscr V_i\otimes\mscr U_i$ and $\mscr U_i$ for all $i=1,\ldots,L$, and find a basis for equivariant weights and invariant biases at level $n_0$ using the algorithm of~\citep{pmlr-v139-finzi21a}. 
Specifically, Theorem 1 in \citep{pmlr-v139-finzi21a} states that $b\in (U_i^{(n_0)})^{G^{(n_0)}}$ if and only if $b$ satisfies that
\begin{equation} \label{eq:free-system}
\begin{aligned}
    &B\, b = 0 \,\,\,\,\, &&\left(\text{for all } B \in \cB_i^{(n_0)}\right),\\ &(D - I)\, b = 0\,\,\,\,\, &&\left(\text{for all } D \in \cD_i^{(n_0)}\right),
\end{aligned}\end{equation}
where $\cB_i^{(n_0)}$ is a basis for the Lie algebra of $G^{(n_0)}$ and $\cD_i^{(n_0)}$ are discrete generators for $G^{(n_0)}$, which are finite sets.
Here $B$ and $D$ are represented as matrices acting on $U_i^{(n_0)}$. 
Since equivariant linear maps $W\colon V_i^{(n_0)}\to U_i^{(n_0)}$ are just the invariants in $\mc L(V_i^{(n_0)},U_i^{(n_0)})$, they constitute the kernel of the analogously-defined set of equations.
Thus, finding a basis for equivariant weights and invariant biases reduces to finding a basis for the kernel of a matrix~\eqref{eq:free-system}, which is typically large and sparse. 
This can be tackled either using a Krylov-subspace method~\citep{pmlr-v139-finzi21a} or a sparse LU decomposition~\citep{computing_nullSpace,spspaces}. 

\textbf{Finding bases for Compatible NN.} If we rather want to find basis for layers satisfying~\eqref{eq:compatibleNN}, 
we fix $n_0$ \emph{exceeding the presentation degrees} of $\mscr V_i$ and let $d_i$ be its generation degree. For most representations, \eqref{eq:compatibleNN} implies zero bias, so we focus on the weights.
We find a basis for weights $W_i^{(n_0)}$ satisfying the hypotheses of Theorem~\ref{thm:ext_to_morph} by noting that Assumption~\ref{ass:comp} holds if and only if $W_i^{(n_0)}$ is $G^{(n_0)}$-equivariant and satisfies
\begin{equation}
\left[\cP_{V_i^{(m)}} \otimes \left(I - \cP_{U_i^{(m)}}\right)\right] \, \mathrm{vec}\left( W_i^{(n_0)}\right) = 0,
\end{equation}
for all $m \leq d_i$,
where vec denotes the vectorization of a matrix by stacking its columns. 
This characterization allows us to find a basis for weights by finding a basis for the kernel of a sparse matrix.


\textbf{Extending to arbitrary dimensions.} 
For any $n>n_0$, we extend our trained network at level $n_0$ by finding the unique equivariant weights and biases at level $n$ projecting onto the trained ones as in~\eqref{eq:freeNN}. Formally, we find the unique $W_i^{(n)}$ and $b_i^{(n)}$ satisfying~\eqref{eq:free-system} with $n_0$ replaced by $n$, and 
\begin{equation}\label{eq:ext_lin_syst}\tag{\texttt{ExtSyst}}
\begin{aligned}
    &\begin{bmatrix} \mc P_{V_i^{(n_0)}}\otimes \mc P_{U_i^{(n_0)}}\end{bmatrix}\,\mathrm{vec}\left(W_i^{(n)}\right) = \mathrm{vec}\left(W_i^{(n_0)}\right),\\
     &\mc P_{U_i^{(n_0)}}\, b_i^{(n)} =  b_i^{(n_0)}.
\end{aligned}\end{equation}
This amounts to solving a linear system, which is again typically sparse. It can be solved 
via iterative methods, e.g., stochastic gradient descent or LSQR~\citep{paige1982lsqr}. For any $n < n_0$, we set $W_i^{(n)}$ and $b_i{(n)}$ via~\eqref{eq:freeNN}, which is equivalent to \eqref{eq:ext_lin_syst}
 with $n$ and $n_0$ swapped.

\begin{algorithm}[t]
	\caption{Train a freely-described (resp., compatible) neural network.}
	\label{algo:free_descr}
	\begin{algorithmic}
		\State \textbf{Input:} Architecture $\mscr V_i,\mscr U_i$, group sequence $\{\Gn\}$, and training data at level $n_0$. 
\vspace{.2cm}
        \State \textbf{Stage 1} Find free (resp., compatible) bases of the linear layers at level $n_0$. 
        \State \phantom{\hbox{\textbf{Stage 1}}} Optimize over the coefficients in the bases to train the network at level $n_0$.
\vspace{.2cm}
        \State \textbf{Stage 2} For any $n > n_0$, solve~\eqref{eq:ext_lin_syst} to extend the trained network to level $n$.
        \State \phantom{\hbox{\textbf{Stage 2}}} For any $n < n_0$, project via~\eqref{eq:freeNN} to restrict the trained network to level $n$.
	\end{algorithmic}
\end{algorithm}

%% file: sections/experiments.tex
\section{NUMERICAL EXPERIMENTS}\label{sec:experiments}

\begin{figure*}[t!]
    \centering
    \begin{subfigure}[t]{0.32\linewidth}
        \centering
        \qquad Trace
        \vspace{0.5em}
        \includegraphics[width=\linewidth]{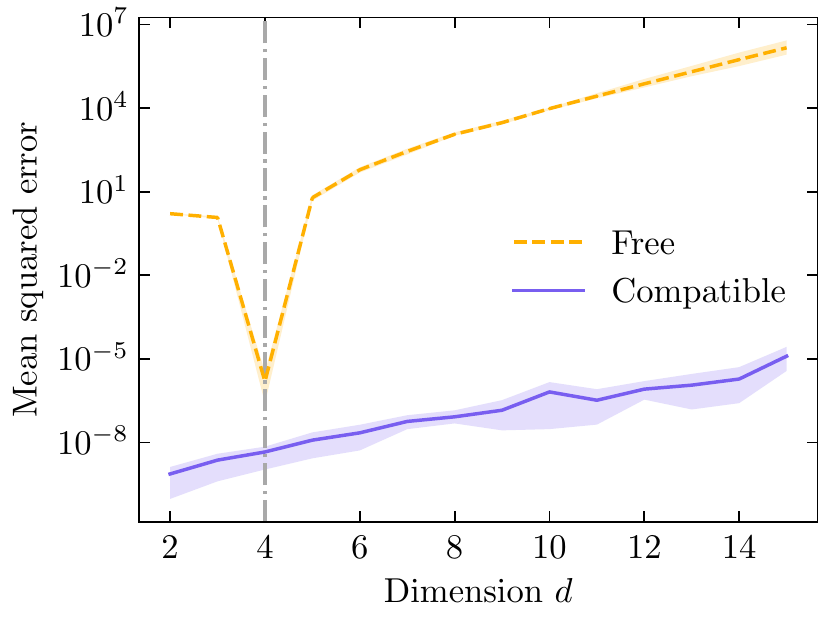}
    \end{subfigure}%
    \hfill
    \begin{subfigure}[t]{0.32\linewidth}
        \centering
        \qquad Diagonal extraction
        \vspace{0.5em}
        \includegraphics[width=\linewidth]{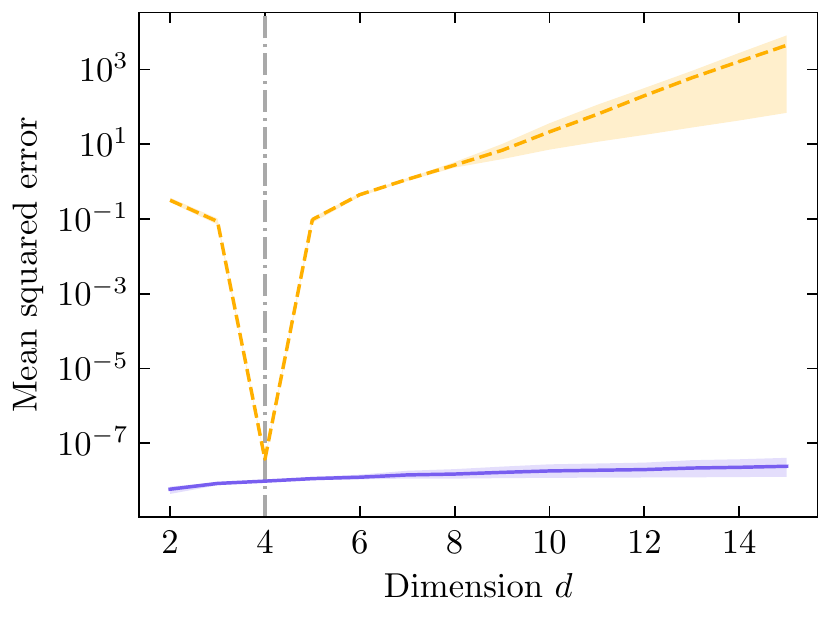}
    \end{subfigure}
    \hfill
    \begin{subfigure}[t]{0.32\linewidth}
        \centering
        \qquad Symmetric projection
        \vspace{0.5em}
        \includegraphics[width=\linewidth]{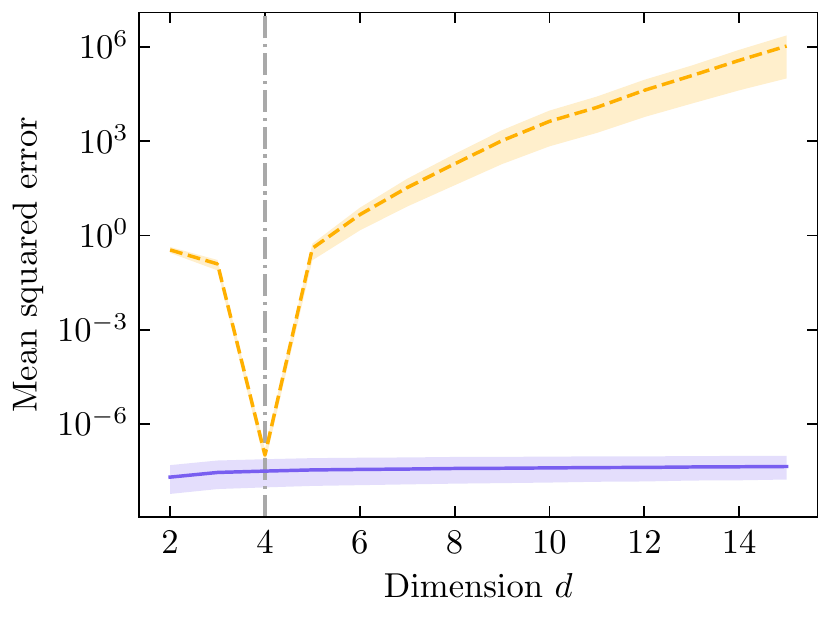}
    \end{subfigure}%
    \\ \centering 
    \begin{subfigure}[t]{0.32\linewidth}
        \centering
        \qquad Top singular vector
        \includegraphics[width=\linewidth]{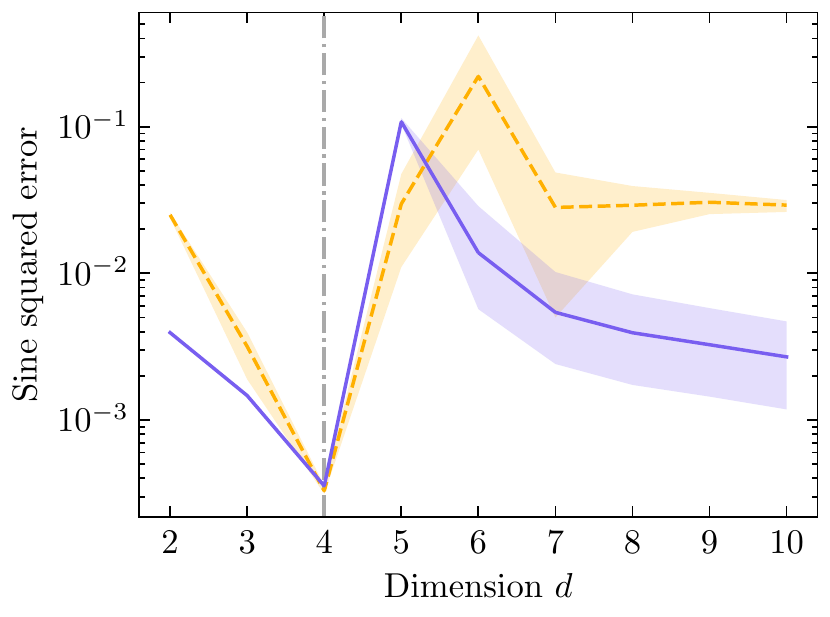}
    \end{subfigure}
    \quad
    \begin{subfigure}[t]{0.32\linewidth}
        \centering
        \quad Orthogonal invariance task
        \includegraphics[width=\linewidth]{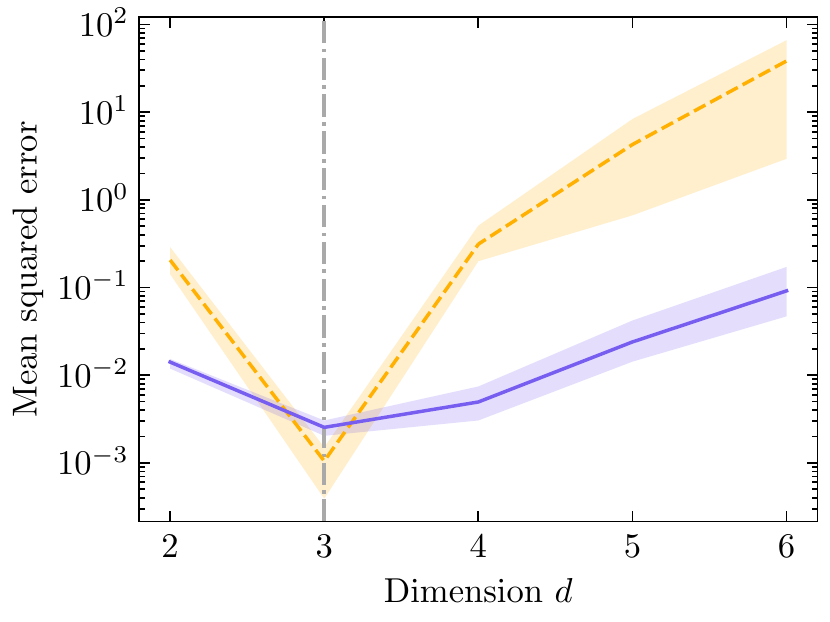}
    \end{subfigure}
    \caption{Test errors across dimensions for free and compatible networks. Each experiment is run three times; the lighter bands show the max and min runs, while the bold line shows the average. For diagonal extraction and symmetric projection, we measure the average MSE per entry. Vertical gray lines mark the dimension used for learning.}\label{fig:errors}
\end{figure*}
In this section, we use Algorithm~\ref{algo:free_descr} to learn some of the compatible mappings from Section~\ref{sec:compatibility}. Our implementation is based on that of~\citep{pmlr-v139-finzi21a}, and is available at 
\begin{center}
    \mbox{\url{https://github.com/mateodd25/free-nets}.}
\end{center}
For each experiment, we aim to learn a sequence of mappings $\{\fn\colon \Vn_1\to \Vn_{L+1}\}$. To do so, we fix a level $n_0$ and randomly generate data $\{(X_i,f^{(n_0)}(X_i))\}\subseteq V_1^{(n_0)}\times V_{L+1}^{(n_0)}$. To fit the data, we consider the architectures described in the Appendix~\ref{sec:app-description} and compute bases for weights and biases as we describe in Section~\ref{sec:computational}. We then optimize over the coefficients in those bases using ADAM~\citep{kingma2014adam} starting from random initialization. 
Finally, we extend the trained network at level $n_0$ to several levels $n$. We evaluate our extended network on random test data generated at each level from the ground-truth sequence of maps. 

We consider five examples from Section~\ref{ex:comp}: the trace $f(X) = \Tr(X),$ 
diagonal extraction 
$f(x) = \diag(\diag(X))$, 
symmetric projection
$f(X) = (X+X^\top)/2$, 
the top right-singular vector 
$f(X) = \argmax_{\|v\|=1}\|Xv\|$, and the function defined in \eqref{eq:sin}. The first and last examples are invariant functions; the rest are equivariant mappings. For the first four examples, we set $\Gn = \mathrm{S}_n$, and for the last one, we set $\Gn = \mathrm{O}(n).$ We train with $3000$ data points generated randomly and evaluate the test error at each dimension on $1000$ fresh samples. 

We use the mean squared error (MSE) as our loss function for all the experiments except for learning the top singular vector. For that experiment, we use the squared sine loss proposed by \citep{maron2018invariant}, i.e., $\ell(\widehat{y}, y) = 1 - \dotp{\widehat{y} ,y}^2/\|\widehat{y}\|^2 \|y\|^2$. The resulting errors in each dimension for both free and compatible NNs are shown in Figure~\ref{fig:errors}. The test errors at the trained dimension are competitive with those obtained in~\citep{maron2018invariant,pmlr-v139-finzi21a, villar2021scalars} for learning the same mappings, see Table~\ref{tab:err_comparison} for comparison. Moreover, the errors we obtain for higher-dimensional extensions are competitive with errors obtained in the literature for NNs trained at those higher dimensions. For the orthogonally invariant task example, we trained our compatible NN in 3 dimensions using 3000 data points and got an error of 2e-2 when extending to 5 dimensions. In contrast, the NN of~\citep{villar2021scalars}, which is tailored to orhogonal equivariance, achieves an error of 9e-3 when trained in 5 dimensions on the same number of points, while the NN of~\citep{pmlr-v139-finzi21a} trained in the same dimension achieves 4e-2. 

\begin{table}[h]
    \centering
    \begin{tabular}{@{}lll@{}}
        \toprule
        \textbf{Example} & \textbf{Ours} & \textbf{Previous} \\ \toprule
        trace & 3e-8 & 1e-3~\citep{maron2018invariant} \\ \midrule
        diag & 1e-8 & 7e-6~\citep{maron2018invariant}\\ \midrule
        sym & 3e-8 & 7e-6~\citep{maron2018invariant} \\ \midrule
        svd & 4e-4 & 1.6e-3~\citep{maron2018invariant} \\ \midrule
        \multirow{2}{*}{orth.} & \multirow{2}{*}{2e-3} & 1e-2~\citep{pmlr-v139-finzi21a};\\ & & 5e-4~\citep{villar2021scalars} \\ \bottomrule
    \end{tabular}
    \caption{Comparison of errors from the literature}\label{tab:err_comparison}
\end{table}

 Due to memory limitations, we set the training level to $n_0 = 3$ for the orthogonal invariance task, whereas Algorithm~\ref{algo:free_descr} would require $n_0=6$ to uniquely extend free networks\footnote{ 
 We find the minimum norm solution for the weights and biases using LSQR~\citep{paige1982lsqr}.}  since the presentation degree of the linear maps between the hidden layers equal to six\,---\, see Theorem~\ref{thm:correctness} in the appendix.
 This highlights an advantage of imposing compatibility\,---\, it allows us to uniquely extend a trained network from lower-dimensional data. Moreover, we see that imposing our compatibility condition yields substantially lower errors across dimensions. Remarkably, the test error for free NNs increases by many orders of magnitude for simple functions, 
a phenomenon that was previously noted by~\citep{maron2018invariant}.
This further underscores the importance of compatibility when generalizing to other dimensions.

%% file: sections/limitations.tex
\section{LIMITATIONS}\label{sec:limitations}
\paragraph{Extension to non-compact groups.}
We assumed that the group is compact and acts orthogonally so that the orthogonal projection $\mathcal{P}_{\Vn} = \phin^*$ is $\Gn$-equivariant, used for instance in the proof of Proposition~\ref{prop:surj_imp_inj}. 
This property holds more generally when $\Gn$ is linearly reductive.
Thus, our framework extends to non-compact groups such as the Lorentz group $\mathrm{O}(3,1)$ and so-called rescaling groups $\RR_{>0}^n$ consisting of diagonal matrices with strictly positive entries, both of which are important to applications of equivariant deep learning to physics~\cite{pmlr-v119-bogatskiy20a,JMLR:v24:22-0680}.

Some groups arising in applications are not reductive however, including any group containing translations such as the group of rigid motions $\mathrm{SE}(n)$. 
This difficulty can be circumvented in some cases, such as centering the data when learning an $\mathrm{SE}(n)$-invariant function and imposing only $\mathrm{SO}(n)$-invariance on the architecture. In general, however, more work is needed to address the non-reductive case. 

\paragraph{Compatible activation functions and biases.}
Our compatibility conditions~\eqref{eq:compatibleNN} require the activation functions we use to be compatible with our embeddings relating the different dimensions. 
For example, applying any activation function $\sigma\colon\RR\to\RR$ entrywise always satisfies~\eqref{eq:compatibleNN} for the graphon embeddings. 
However, only activations with $\sigma(0)=0$ satisfy~\eqref{eq:compatibleNN} for the zero-padding embeddings. 
On the one hand, some activation functions used in practice, such as the sigmoid $\sigma(x)=(1+e^{-x})^{-1}$, do not satisfy this condition. 
Moreover, compatibility requires the biases to be the same across dimensions, which implies that the biases must be multiples of the first canonical basis vector when using the zero-padding embedding, and a multiple of the all-1's matrix when using the graphon embedding.
Therefore, compatibility limits the architectures that we can use. 
On the other hand, if the sequence of mappings we are trying to learn is known in advance to satisfy compatibility, then it is desirable to use architectures that enforce compatibility as well. 
In that case the mapping we learn also generalizes correctly to other dimensions, as the numerical experiments in Section~\ref{sec:experiments} demonstrate.

\paragraph{Scalability.}

We hope to address several limitations of our approach in future work. 
First, while our compatibility condition improves generalization, it can be restrictive in certain settings. For example, it often yields zero bias. It would be interesting to learn the relation between maps in different dimensions directly from the data rather than assume it is known in advance. 
Second, for examples like the leading eigenvector, the error increases substantially when extending to higher dimensions, even if we impose compatibility. It would be interesting to improve generalization to higher dimensions in these examples by 
imposing additional compatibility conditions. 
Third, extending a trained network to higher dimensions in our examples involved solving large and sparse systems, which we currently do not fully exploit. Incorporating sparsity in the toolbox of~\citep{pmlr-v139-finzi21a} will enable training larger 
networks.

%% file: sections/conclusions.tex
\section{CONCLUSIONS AND FUTURE WORK} \label{sec:conclusions}
We leveraged representation stability to prove that a broad family of equivariant NNs extends to higher dimensions, which enables us to train an infinite sequence of NNs using a finite amount of data. Extending networks to higher dimensions 
often results in substantial test errors. 
To improve generalization,
we introduced a compatibility condition relating networks across dimensions. We characterized compatibility and developed an algorithm to train free and compatible NNs. Finally, we applied our method to several numerical examples from the literature. In these examples, free NNs trained without imposing compatibility generalize poorly to higher dimensions, even when learning simple linear functions. In contrast, compatible NNs generalize significantly better.

As highlighted in Section~\ref{sec:limitations}, our current implementation faces scalability challenges, primarily due to the inadequate support for GPU-accelerated sparse linear algebra operations in Python. Notably, there are already some efforts to close this gap~\cite{potapczynski2023cola}. Moving forward, we are optimistic about enhancing our framework to leverage these advancements effectively. As we also mentioned in Section~\ref{sec:limitations}, the compatibility condition constraints the family of functions one could hope to learn. This presents a compelling avenue for future research: exploring other regularization strategies that lead to good generalization across dimensions. Moreover, the concepts introduced in this paper paved the way to studying any-dimensional statistical learning, raising several natural questions: Can we bound the generalization error across dimensions? How do we characterize the notation of statistical complexity for classes of any-dimensional problems? And how do we approach the formulation of lower bounds within this context?

%% file: sections/appendix.tex
\section{MORE EXAMPLES OF CONSISTENT SEQUENCES}
We let $\{\phin \colon \RR^n \rightarrow \RR^{n+1}\}$ denote the sequence of zero-padding embeddings.
\begin{itemize}[leftmargin=2mm]
\item[] \textit{Signed permutation sequences.} Let $B_n$ (resp. $D_n$) be the signed permutation group (resp. even-signed permutation group) in $n$ elements. Consider the sequence of representations $\Vn = \RR^n$ with the standard action of permuting and flipping the signs of the coordinates. Then, $\mscr V = \{(\Vn, \varphi)\}$ is a consistent sequence with generation and presentation degrees equal to one. 
 \item[] \textit{Orthogonal rotation sequences.} Let $O(n)$ (resp., $SO(n))$ be the orthogonal group (resp., special orthogonal group) in dimension $n$. Consider the sequence of representations $\Vn = \RR^n$ with the standard rotation action. Then, $\mscr V = \{(\Vn, \varphi)\}$ is again a consistent sequence with generation degree equal to one. Characterizing the presentation degree is an open question. Yet, numerically, the presentation degree for $S_n$ seems to work well. 
 \item[] \textit{Lorentz sequences.} Let $O(1, n)$ be the Lorentz group (resp., special orthogonal group) in dimension $n + 1$. Again, consider the sequence of representations $\Vn = \RR^{n+1}$ with the natural action of the Lorentz group. Then, $\mscr V = \{(\Vn, \varphi)\}$ is once more a consistent sequence with generation degree equal to one.
 \item[] \textit{Tensor products.} For any of the above sequences, we can take tensor powers $\mscr V^{\otimes k}=\{(\Vn^{\otimes k},\phin^{\otimes k})\}$ to obtain sequences of tensors. Here $\phin^{\otimes k}$ zero-pads a $n\times\cdots\times n$ tensor $T\in (\RR^n)^{\otimes k}$ to a $(n+1)\times\cdots\times(n+1)$ tensor, and the group acts on rank-1 tensors by $g\cdot(x_1\otimes\cdots\otimes x_k)=(gx_1)\otimes\cdots\otimes(gx_k)$ which extends to an action on general tensors by linearity. For a example, a permutation $g\in \mathrm{S}_n$ acts on $T\in(\RR^n)^{\otimes k}$ by $(g\cdot T)_{i_1,\ldots,i_k}=T_{g(i_1),\ldots,g(i_k)}$. 

 \item[] \textit{Direct sums.} For any of the above sequences, we can take direct sums $\mscr V^{\oplus k}=\{(\Vn^{\oplus k},\phin^{\oplus k})\}$. Here $\phin^{\oplus k}$ zero-pads each of the $k$ vectors in an element of $(\RR^n)^{\oplus k}$, and the group acts on each vector in such an element by $g\cdot(x_1,\ldots,x_k)=(g\cdot x_1,\ldots,g\cdot x_k)$.
\end{itemize}

\section{PRESENTATION DEGREE}\label{sec:presentation-degree}
In this section, we give a definition of the presentation degree.
In what follows, we introduce a series of preliminary definitions. We highlight that the concepts in this section are technical in nature and require a certain familiarity with abstract algebra. 
We refer the interested reader to~\cite[\S4.1]{levin2023free} for a more accessible presentation of this material motivated by the problem of extending an equivariant map to a compatible sequence. 
For this section, we drop the boldface notation since we need to use additional superscripts. 

\begin{definition}[Stabilizing subgroups]
    If $\mscr V = \{\Vnn\}$ is a consistent sequence of $\{\Gnn\}$-representations, define its \emph{stabilizing subgroups} $\{H^{(n,d)}\}_{d\leq n}$ by
    \begin{equation*}
        H^{(n,d)} = \{g\in \Gnn: g\cdot v = v \textrm{ for all } v\in V^{(d)}\}.
    \end{equation*}
\end{definition}
In the context of $\Gnn = S_n$ and $\Vnn = \RR^n$, the subgroup $H^{(n,d)}$ corresponds to the permutations in $n$ elements that leave fixed the first $d$ components.
\begin{definition}[Modules]
    Let $\mscr V$ and $\mscr U = \{\Unn\}$ be consistent sequences of $\{\Gnn\}$-representations and let $\{H^{(n,d)}\}$ be the stabilizing subgroups of $\mscr V$. Then $\mscr U$ is a \emph{$\mscr V$-module} if $U^{(d)}\subseteq (\Unn)^{H^{(n,d)}}$ for all $d\leq n$. 
\end{definition}
Recall that if $H\subseteq G$ is a subgroup and $U$ is an $H$-representation, the induced representation of $U$ to $G$ is $\mathrm{Ind}_H^GU = \RR[G]\otimes_{\RR[H]}U$. If $G\subseteq G'$ and $H\subseteq H'$ are subgroups satisfying $H'\cap G=H$ and $U$ is an $H'$-representation, we have an inclusion $\mathrm{Ind}_H^GU\subseteq \mathrm{Ind}_{H'}^{G'}U$. 
\begin{definition}[Induction and algebraically free sequences]\label{def:free_mods}
Let $\mscr V$ by a consistent sequence of $\{\Gnn\}$-representations, and for $d\leq n$ let $H^{(n,d)}\subseteq \Gnn$ be its stabilizing subgroups. 
\begin{enumerate}[(a),leftmargin=8mm]
    \item Fix $d\in\NN$, and a $G^{(d)}$-representation $U^{(d)}$.
    Define the \emph{$\mscr V$-induction sequence} 
    \begin{equation*}
        \mathrm{Ind}_{G^{(d)}}(U^{(d)}) = \left\{\mathrm{Ind}_{G^{(d)}H^{(n,d)}}^{\Gnn}U^{(d)}\right\}_n,
    \end{equation*}
    where the induced representation is taken to be 0 when $n<d$, and $G^{(d)}H^{(n,d)}=\{gh: g\in G^{(d)},\ h\in H^{(n,d)}\}$ is the subgroup generated by $G^{(d)}$ and $H^{(n,d)}$ inside $\Gnn$. The $\mscr V$-induction sequence is a $\mscr V$-module by construction. 

    \item A consistent sequence $\mscr F$ is an \emph{algebraically free $\mscr V$-module} if it is a direct sum of $\mscr V$-induction sequences. The sequence $\mscr V$ itself is algebraically \emph{free} if it is an algebraically free $\mscr V$-module.
\end{enumerate}
\end{definition}
Since $ghg^{-1}\in H^{(n,d)}$ for any $g\in G^{(d)}$ and $h\in H^{(n,d)}$, the subgroup generated by $G^{(d)}$ and $H^{(n,d)}$ inside $\Gnn$ is the set of products $G^{(d)}H^{(n,d)}$. As we have inclusions $G^{(d)}H^{(n,d)}\subseteq G^{(d)}H^{(n+1,d)}$ and $\Gnn\subseteq \Gnnp$, and moreover $G^{(d)}H^{(n+1,d)}\cap \Gnn = G^{(d)}H^{(n,d)}$, we have inclusions $\mathrm{Ind}_{G^{(d)} H^{(n,d)}}^{\Gnn}U^{(d)} \subseteq \mathrm{Ind}_{G^{(d)} H^{(n+1,d)}}^{\Gnnp}U^{(d)}$. This shows that $\mathrm{Ind}_{G^{(d)}}(U^{(d)})$ is indeed a consistent sequence.

\begin{example}
Let $\mscr V=\{\RR^n\}$ with embeddings by zero-padding and the action of $\Gnn=\mathrm{S}_n$ from Section~\ref{sec:free-descriptions}. Then $\mscr V = \mathrm{Ind}_{G^{(1)}}V^{(1)}$, hence it is algebraically free.
\end{example}

Any consistent sequence is a quotient of an algebraically free one. To formalize this, we first define an appropriate notion of maps between consistent sequences.
\begin{definition}[Morphisms of sequences]
    If $\mscr V = \{\Vnn\}$ and $\mscr U=\{\Unn\}$ are consistent sequences of $\{\Gnn\}$-representations, then a \emph{morphism of sequences} $\mscr V\to\mscr U$ is a sequence of equivariant linear maps $\{\Wnn\in \mc L(\Vnn,\Unn)^{\Gnn}\}$ satisfying $\Wnnp|_{\Vnn} = \Wnn$. 
\end{definition}
Note that~\eqref{eq:compatibleNN} precisely requires the sequences of weights in a compatible NN to define a morphism of sequences. 
To write a general consistent sequence $\mscr V=\{V^{(n)}\}$ as a quotient of a free one, suppose $\mscr V$ is generated in degree $d$. Define the algebraically free $\mscr V$-module $\mscr F = \bigoplus_{i=1}^d\mathrm{Ind}_{G^{(i)}}(V^{(i)}) = \{F^{(n)}\}$ and consider the sequence of linear maps $\Wnn\colon F^{(n)}\to V^{(n)}$ sending $g\otimes v\mapsto g\cdot v$ for each $g\otimes v\in \mathrm{Ind}_{G^{(i)}}(V^{(i)})$. It is easy to check that $\{\Wnn\}\colon\mscr F\to\mscr V$ is a morphism of sequences and that $\Wnn$ is surjective for any $n\geq d$, as the image of $\Wnn$ is precisely $\RR[\Gnn]V^{(\min\{n, d\})} = V^{(n)}$.
The sequence of kernels of such a morphism of sequence $\{\Wnn\}$ is also a consistent sequence.
\begin{proposition}
    If $\mscr V=\{(\Vnn,\phin)\}$ and $\mscr U = \{\Unn\}$ are consistent sequences and $\{\Wnn\}\colon\mscr V\to\mscr U$ is a morphism of sequences, then $\ker\{\Wnn\}=\{(\ker \Wnn,\phin)\}$ is a consistent sequence.
\end{proposition}
\begin{proof}
    Since $\Wnn$ is $\Gnn$-equivariant, its kernel is a $\Gnn$-subrepresentation of $\Vnn$. The embeddings $\phin$ remain $\Gnn$-equivariant isometries when restricted to $\ker \Wnn$.
\end{proof}
We are now ready to define the presentation degree.
\begin{definition}[Presentation degree]
Let $\mscr V_0$ be a consistent sequence of $\{\Gnn\}$-representations. We say that a $\mscr V_0$-module $\mscr V$ is generated in degree $d$, and {presented in degree} $$k=\max\{d,r\}$$ if there exists an algebraically free $\mscr V_0$-module $\mscr F$ generated in degree $d$ and a morphism of sequences $\{\Wnn\}\colon\mscr F\to\mscr V$ such that $\Wnn$ is surjective for all $n$ and $\ker\{\Wnn\}$ is generated in degree $r$. The \emph{presentation degree} is the smallest such $k \in \NN$. 
\end{definition}
Note that the presentation degree is always at least as large as the generation degree and that the two are equal for algebraically free sequences.
As stated in Section~\ref{sec:free-descriptions}, the projections in Proposition~\ref{prop:surj_imp_inj} become isomorphisms starting from the presentation degree. 
The following proof of this fact, usually stated in the representation stability literature in terms of \emph{coinvariants}~\cite[\S3]{FImods}, appeared first in~\citep{levin2023free}. 
\begin{proposition}\label{prop:presentation_implies_isom_of_canon}
    Let $\mscr V_0$ be a consistent sequence of $\{\Gnn\}$-representations and $\mscr V$ be a $\mscr V_0$-module presented in degree $k$. Then, the maps $\mc P_{\Unn}\colon (\Vnnp)^{\Gnnp}\to (\Vnn)^{\Gnn}$ are isomorphisms for all $n\geq k$.
\end{proposition}
\begin{proof}
    As $\mscr V$ is presented in degree $k$, there exists an algebraically free $\mscr V_0$-module $\mscr F=\{\Fnn\}$ and a surjective morphism $\{\Wnn\}\colon\mscr F\to \mscr V$ such that both its kernel $\mscr K=\{\Knn\}$ and $\mscr F$ itself are generated in degree $k$. 
    Because each map $\Fnn\to \Vnn$ is a $\Gnn$-equivariant surjection with kernel $\Knn$, its restriction to invariants $(\Fnn)^{\Gnn}\to (\Vnn)^{\Gnn}$ is surjective with kernel $(\Knn)^{\Gnn}$. 

    As $\mscr F$ is an algebraically free $\mscr V_0$-module, there exist integers $d_j$ and $G^{(d_j)}$-representations $U^{(d_j)}$ satisfying $\mscr F = \bigoplus_j\mathrm{Ind}_{G^{(d_j)}}U^{(d_j)}$. Such $\mscr F$ has generation degree $\max_jd_j\leq k$. Therefore, letting $\{H^{(n,d)}\}$ be the stabilizing subgroups of $\mscr V_0$, we have for $n\geq k$
    \begin{align*}
        \left(\Fnn\right)^{\Gnn} & = \bigoplus\nolimits_j\Big(\mathrm{Ind}_{G^{(d_j)}H^{(n,d_j)}}^{\Gnn}(U^{(d_j)})\Big)^{\Gnn} \\ 
        & \cong \bigoplus\nolimits_j(U^{(d_j)})^{G^{(d_j)}}.  
    \end{align*}
    The last isomorphism follows from the fact that for any groups $H\subseteq G$ and any $H$-representation $U$, we have an isomorphism $\left(\mathrm{Ind}_H^G(U)\right)^G\cong U^H$ given by sending $u\in U^H$ to $\sum_{i=1}^kg_i\otimes u\in \left(\mathrm{Ind}_H^G(U)\right)^G$ where $g_1=\mathrm{id},g_2,\ldots,g_k$ are coset representatives for $G/H$.
    Thus, $\dim\, (\Fnn)^{\Gnn}$ is constant for $n\geq k$.
    Moreover, by Proposition~\ref{prop:surj_imp_inj} and the fact that $\mscr K$ and $\mscr U$ are generated in degree $k$, we have $\dim\, (\Knn)^{\Gnn} \geq \dim\, (\Knnp)^{\Gnnp}$ and $\dim\, (\Vnn)^{\Gnn}\geq \dim\, (\Vnnp)^{\Gnnp}$ for all $n\geq k$. 
    
    By the rank-nullity theorem, we have $\dim (\Vnn)^{\Gnn} = \dim (\Fnn)^{\Gnn} - \dim\, (\Knn)^{\Gnn}$. As $\dim (\Fnn)^{\Gnn}$ is constant while both $\dim (\Vnn)^{\Gnn}$ and $\dim\, (\Knn)^{\Gnn}$ are nonincreasing for $n\geq k$, we conclude that they are all constant for $n\geq k$. 
    To conclude, note that $\mc P_{\Vnn}$ is injective for all $n\geq k$ by Proposition~\ref{prop:surj_imp_inj}.
\end{proof}

\section{PROOF OF THEOREM~\ref{thm:ext_to_morph}} \label{sec:proof-compatibility}
The following proof first appeared in~\citep{levin2023free}.
Since $\mscr V,\mscr U$ are obtained from direct sums and tensor products of $\mscr V_0$, they are both $\mscr V_0$-modules. 
Let $\{H^{(n,d)}\}$ be the stabilizing subgroups of $\mscr V_0$.
Suppose first that $\mscr V = \mscr F = \bigoplus_j\mathrm{Ind}_{G^{(d_j)}}F^{(d_j)}$ is free. Note that it is generated in degree $\max_jd_j\leq d_{\rm g}$.
Let $W_{d_j} = W^{(n_0)}|_{F^{(d_j)}}$ and fix $n\geq d_j$.
Because $\mscr U$ is a $\mscr V_0$-module, we have $U^{(d_j)}\subseteq (\Unn)^{H_{n,d_j}}$, so we can view $U^{(d_j)}$ as a representation of $G^{(d_j)}H^{(n,d_j)}$ on which $H^{(n,d_j)}$ acts trivially. 
    
    In general, if $H\subseteq G$ is a subgroup, if $V,U$ are $H$-representations, and if $W\in\mc L(V,U)^H$, we can define $\mathrm{Ind}(W)\colon \mathrm{Ind}_H^GV \to\mathrm{Ind}_H^GU$ by sending $g\otimes v\mapsto g\otimes Wv$ for $g\in G$ and $v\in V$. 
    Also, if $V\subseteq U$ and $V$ is furthermore a $G$-representation, then there is an equivariant map $\mathrm{Ind}_H^GV\to U$ sending $g\otimes v\mapsto g\cdot v$. 
    As $W_{d_j}(F^{(d_j)})\subseteq U^{(d_j)}$ and is $G^{(d_j)}H^{(n,d_j)}$-equivariant, we can combine the above two maps to obtain the following equivariant composition
    \begin{equation*}\begin{aligned}
        \Wnn_j\colon &\mathrm{Ind}_{G^{(d_j)}H^{(n,d_j)}}^{\Gnn}F^{(d_j)} \xrightarrow{\mathrm{Ind}(W_{d_j})} \mathrm{Ind}_{G^{(d_j)}H^{(n,d_j)}}^{\Gnn}U^{(d_j)}\\ &\xrightarrow{g\otimes u\mapsto g\cdot u} \Unn.
    \end{aligned}\end{equation*}
    Note that $W^{(n_0)}_{j}=W^{(n_0)}|_{\mathrm{Ind}_{G^{(d_j)}H^{(n_0,d_j)}}(F^{(d_j)})}$, since $W^{(n_0)}_j(g\otimes f) = g\cdot W^{(n_0)}f$ for all $g\in \Gnn$ and $f\in F^{(d_j)}$. Also, $\{\Wnn_j\}$ defines a morphism $\mathrm{Ind}_{G_{d_j}}(F^{(d_j)})\to \mscr U$. Therefore, the desired extension of $W^{(n_0)}$ to a morphism of sequences $\{\Wnn\}$ is given by $\Wnn = \bigoplus_j\Wnn_j\colon \Vnn\to \Unn$.

    Now suppose $\mscr F$ is an algebraically free $\mscr V$-module as above with a surjection $\mscr F\to \mscr V$ whose kernel $\mscr K=\{\Knn\}$ is generated in degree $d_{\rm p}$. Define the composition 
    \begin{equation*}
        \widetilde W^{(n_0)}\colon F^{(n_0)}\to V^{(n_0)} \xrightarrow{W^{(n_0)}} U^{(n_0)},
    \end{equation*}
    which satisfies $\widetilde W^{(n_0)}(F^(j))\subseteq U^{(j)}$ for all $j\leq d_{\rm g}$ by assumption and $\widetilde W^{(n_0)}(K^{(n_0)})=0$ by its definition. By the previous paragraph, it extends to a morphism $\{\widetilde \Wnn\colon \Fnn\to \Unn\}$. Because $\mscr K$ is generated in degree $d_{\rm p}$ and $n_0\geq d_{\rm p}$, we have $\Knn = \RR[\Gnn]K^{(n_0)}$. Because $\widetilde \Wnn$ is equivariant, we have $\widetilde \Wnn(\Knn)=0$. Therefore, $\widetilde \Wnn$ can be factored as $\Fnn\to \Fnn/\Knn = \Vnn\xrightarrow{\Wnn} \Unn$, where the maps $\Wnn$ in this factorization give the desired extension of $W^{(n_0)}$ to a morphism $\mscr V\to \mscr U$. \qed

\section{CORRECTNESS OF THE COMPUTATIONAL RECIPE}\label{sec:correctness}
In this section, we prove the correctness of Algorithm~\ref{algo:free_descr}. That is, if $n_0$ is set to be large enough, then the algorithm will generate the unique free (or compatible) extension of the network learned at level $n_0.$ It is useful to use the notation $d_{\rm p}(\mscr V)$ to denote the presentation degree of the sequence $\mscr V$.
\begin{theorem}\label{thm:correctness}
    Consider a neural network architecture with consistent sequences $\mscr V_1, \mscr U_1, \dots , \mscr V_L \mscr U_L,$ and $\mscr V_{L+1}$ and consistent activation functions $\sn_1, 
    \dots, \sn_L$ as described Figure~\ref{fig:freeNNs}. Assume we run Algorithm~\ref{algo:free_descr} to learn a neural network at level $n_0$. Then, the following two hold.
    \begin{itemize}[leftmargin=2mm]
        \item[] (\textbf{Free}) If $n_0 \geq \max_i d_{\rm p}(\mscr V_i \otimes \mscr U_i)$ and $n_0 \geq \max_i d_{\rm p}(\mscr U_i)$, then, Algorithm~\ref{algo:free_descr} extends the trained network at level $n_0$ to a unique free neural network.
        \item[] (\textbf{Compatible}) If $n_0 \geq \max_i d_{\rm p}(\mscr V_i)$ and $n_0 \geq \max_i d_{\rm p}(\mscr U_i)$, then, Algorithm~\ref{algo:free_descr} extends the trained network at level $n_0$ to a unique compatible neural network.
    \end{itemize}
\end{theorem}
\begin{proof}
 We start by proving the second statement. This is a simple corollary of Theorem~\ref{thm:ext_to_morph}. Note that once $n_0$ is bigger than the presentation degrees of $\mscr V_i$ and $\mscr U_i$, then, by Theorem~\ref{thm:ext_to_morph}, there is a unique extension of the weights and biases satisfying \ref{eq:compatibleNN}, which corresponds to the unique solution of \eqref{eq:ext_lin_syst}. The first statement follows by an analogous argument by substituting Theorem~\ref{thm:ext_to_morph} with Proposition~\ref{prop:presentation_implies_isom_of_canon}. This completes the proof. 
\end{proof}
\section{DETAILED DESCRIPTION OF THE NUMERICAL EXPERIMENTS} \label{sec:app-description}
\begin{table*}[t] \centering
\begin{tabular}{llll}
\toprule
\multicolumn{1}{c}{\multirow{2}{*}{\textbf{Experiment}}} & \multicolumn{3}{c}{\textbf{Architecture}} \\ \cmidrule{2-4}
\multicolumn{1}{c}{}                                     & Input       & Hidden       & Output       \\ \midrule
Trace                                                    & $\mscr V^{\otimes 2}$             &  $2 \mscr V + 2 \mscr V^{\otimes 2}$             &  $\mscr S$            \\
Diagonal extraction                                      & $\mscr V^{\otimes 2}$            &  $4 \mscr V + 4 \mscr V^{\otimes 2}$             &  $\mscr V^{\otimes 2}$             \\
Symmetric projection                                     & $\mscr V^{\otimes 2}$            &  $4 \mscr V + 4 \mscr V^{\otimes 2}$             &  $\mscr V^{\otimes 2}$    \\
Top singular vector                                      & $\mscr V^{\otimes 2}$            & $25\mscr S + 10 \mscr V + 2 \mscr V^{\otimes 2} +\mscr V^{\otimes 3}$             &  $\mscr V$            \\
Orthogonal invariance                                    & $2 \mscr V$            &  $25\mscr S + 10 \mscr V + 2 \mscr V^{\otimes 2} +\mscr V^{\otimes 3}$   &  $\mscr S$ \\ \bottomrule          
\end{tabular} \vspace{.2cm}
\caption{Architectures used for the numerical examples. Recall that $\mscr S = \{\RR \}$ is the scalar sequence endowed with the trivial action $g \cdot v = v.$}\label{table:arch}
\end{table*}
In this section, we elaborate on the implementation details of the numerical experiments. The code and instructions to reproduce the experiments can be found in \url{https://github.com/mateodd25/free-nets}.
All experiments were run on a 2021 Macbook Air M1 with 16GB of RAM.

\textbf{Initilization.} We initialize the weights and biases at random with small i.i.d Gaussian entries $N(0, 1/100).$

\textbf{Data generation.} The datasets $\left\{\big(X_i, f^{(n)}(X_i)\big)\right\}$ for both training and testing (in every dimension) are generated as i.i.d. standard Gaussian samples with the appropriate size. 

\textbf{Architecture.} For all the experiments, we use three layers, i.e., two hidden layers. All the sequences we take are sums and tensors products of the base sequence $\mscr V = \{\RR^d\}.$ The activation functions are a composition of the form 
$$ { \bm h } \left(\textbf{bilinear}(x) + x\right) $$
where \textbf{bilinear}~\citep{pmlr-v139-finzi21a} is described in Section~\ref{sec:comp_acti} and $\bm h$ changes depending on the group; for the permutation group, $\bm h$ applies entry-wise ReLU for the permutation group, while for the orthogonal group, it applies a gated nonlinearity~\citep{weiler20183d}\,---\,also described in Section~\ref{sec:comp_acti}.

We chose the smallest hidden layer size that yielded competitive results in order to control the size of the linear systems we employ for extending the networks. The top singular vector and orthogonal invariant examples required a larger architecture, so we only extended them to smaller dimensions. The architecture for each example are summarized in Table~\ref{table:arch}.